\documentclass{article}

\usepackage{arxiv}

\usepackage{amsmath,amsfonts,bm,dsfont,nicefrac, amsthm}

\def\eqref#1{equation~\ref{#1}}

\def\1{\bm{1}}

\DeclareMathAlphabet{\mathsfit}{\encodingdefault}{\sfdefault}{m}{sl}
\SetMathAlphabet{\mathsfit}{bold}{\encodingdefault}{\sfdefault}{bx}{n}

\renewcommand{\min}[1]{\underset{#1}{\text{min}}\,}

\renewcommand{\inf}[1]{\underset{#1}{\text{inf}}\,}

\newcommand{\argmin}[1]{\underset{#1}{\text{arg min}}\,}

\newcommand{\expectation}[1]{\underset{#1}{\mathbb{E}}}

\newtheorem{theorem}{Theorem}[section]

\newtheorem{lemma}{Lemma}[section]

\newtheorem{definition}{Definition}[section]

\usepackage[utf8]{inputenc} 
\usepackage[T1]{fontenc}    
\usepackage{hyperref}       
\usepackage{url}            
\usepackage{booktabs}       
\usepackage{amsfonts}       
\usepackage{nicefrac}       
\usepackage{microtype}      
\usepackage{cleveref}       
\usepackage{lipsum}         
\usepackage{graphicx}
\usepackage{natbib}
\usepackage{doi}
\usepackage{tcolorbox}
\usepackage{fontawesome}

\usepackage{graphicx}
\usepackage[acronym]{glossaries}
\usepackage{hyperref}

\usepackage{colortbl}
\usepackage{nicefrac}
\usepackage{subcaption}
\usepackage{booktabs}
\usepackage{xcolor}
\usepackage[linesnumbered,ruled,vlined]{algorithm2e}
\usepackage{lipsum}
\usepackage{wrapfig}
\usepackage{fontawesome}
\usepackage{multirow}
\usepackage{float}
\usepackage{wrapfig}

\DontPrintSemicolon

\definecolor{myorange}{rgb}{0.906,0.435,0.317}
\definecolor{myblue}{rgb}{0.0,0.314,0.408}

\SetKwComment{Comment}{\color{green!50!black}// }{}

\SetKwProg{Function}{function}{}{}

\newacronym{ot}{OT}{Optimal Transport}
\newacronym{mmd}{MMD}{Maximum Mean Discrepancy}
\newacronym{sgd}{SGD}{Stochastic Gradient Descent}
\newacronym{kddm}{KD$^{2}$M}{Knowledge Distillation through Distribution Matching}

\newacronym{kd}{KD}{Knowledge Distillation}

\title{KD$^{2}$M: A unifying framework for feature knowledge distillation}

\date{}

\author{ \href{https://orcid.org/0000-0003-3850-4602}{\includegraphics[scale=0.06]{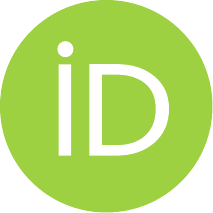}\hspace{1mm}Eduardo Fernandes Montesuma} \\
Sigma Nova\\
Paris, France \\
\texttt{eduardo.montesuma@sigmanova.ai} \\
}

\renewcommand{\shorttitle}{Knowledge Distillation through Distribution Matching}

\hypersetup{
pdftitle={Knowledge Distillation through Distribution Matching},
pdfsubject={stats.ML},
pdfauthor={Eduardo Fernandes Montesuma},
pdfkeywords={Knowledge Distillation, Optimal Transport, Computational Information Geometry, Deep Learning},
}

\begin{document}
\maketitle

\begin{abstract}
Knowledge Distillation (KD) seeks to transfer the knowledge of a teacher, towards a student neural net. This process is often done by matching the networks' predictions (i.e., their output), but, recently several works have proposed to match the distributions of neural nets' activations (i.e., their features), a process known as \emph{distribution matching}. In this paper, we propose an unifying framework, Knowledge Distillation through Distribution Matching (KD$^{2}$M), which formalizes this strategy. Our contributions are threefold. We i) provide an overview of distribution metrics used in distribution matching, ii) benchmark on computer vision datasets, and iii) derive new theoretical results for KD.

\begin{tcolorbox}

This paper was accepted at the 7th International Conference on Geometric Science of Information. Our code is available at:

\begin{center}
\faGithub\,\,\,\url{https://github.com/eddardd/kddm}
\end{center}
\end{tcolorbox}
\end{abstract}


\section{Introduction}

\gls{kd}~\cite{hinton2015distilling} is a problem within machine learning, which transfers knowledge from a large teacher model to a smaller student model~\cite{gou2021knowledge}.

\begin{figure}[ht]
    \centering
    \includegraphics[width=\linewidth]{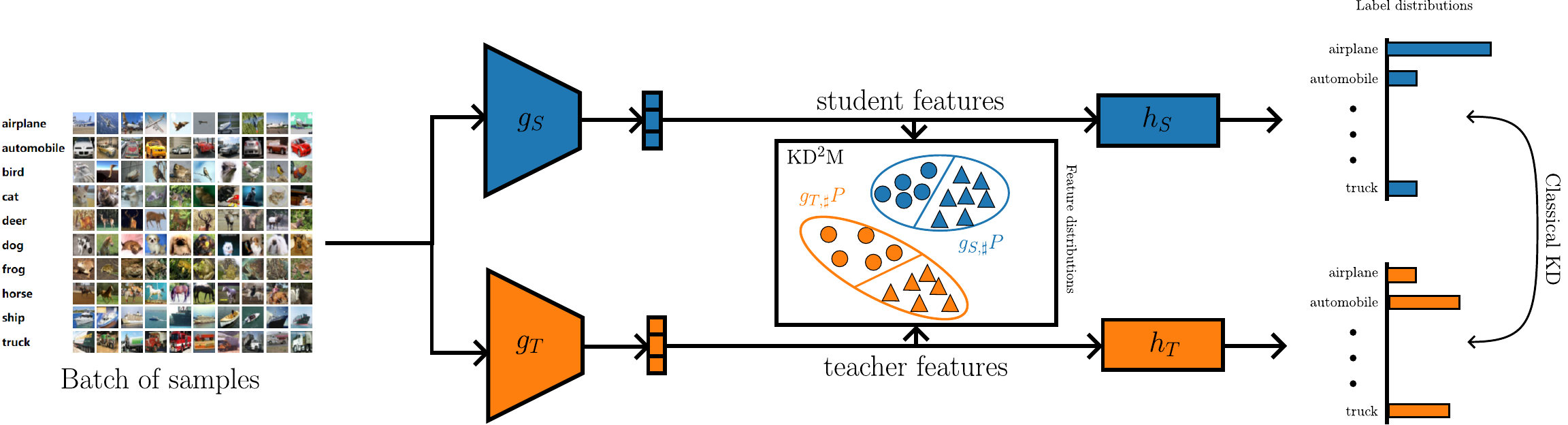}
    \caption{Overview of our knowledge distillation framework, where a batch of samples is fed through a student and a teacher network. First, our framework \gls{kddm} performs distillation by matching the distribution of student's and teacher's features. Second, classical \gls{kd} works by comparing the distributions of student's and teacher's predictions.}
    \label{fig:kddm}
\end{figure}

In this context, \gls{kd} has been closely linked with ideas in information geometry~\cite{nielsen2020elementary} and metrics between probability distributions. Indeed, in their seminal work~\cite{hinton2015distilling}, Hinton, Vinyals, and Dean proposed to distill the knowledge of a teacher toward a student by penalizing students' that generate predictions, i.e., probability distributions over labels, that differ from the teacher. However, this discrepancy only considers \emph{output distribution mismatch}. Recently, a plethora of works~\cite{huang2017like,chen2021wasserstein,lohit2022model,lv2024wasserstein} have considered computing discrepancies in distribution between the activation maps of neural nets, that is, features. We summarize these ideas in Fig.~\ref{fig:kddm}.


In this paper, we propose a general framework for \gls{kd} via matching student and teacher feature distributions,
\begin{align}
    \theta^{\star} = \argmin{\theta \in \Theta} \expectation{(\mathbf{x}^{(P)}, y^{(P)}) \sim P}[\mathcal{L}(y^{(P)}, h_{S}(g_{S}(\mathbf{x}^{(P)})))] + \lambda \mathbb{D}(g_{S,\sharp}P, g_{T,\sharp}P),\label{eq:kd_dm}
\end{align}
where $g_{S}:\mathcal{X} \rightarrow \mathcal{Z}$ and $h_{S}:\mathcal{Z}\rightarrow \mathcal{Y}$ are the student's encoder and classifier (resp. teacher), and $g_{S,\sharp}P$ denotes the \emph{push-forward distribution} of $P$ via $g_{S}$ (cf. equation~\ref{eq:definitions}), and $\mathbb{D}$ is a distribution metric or discrepancy. Here, $\mathcal{X}$, $\mathcal{Z}$ and $\mathcal{Y}$ are called data, feature and, label spaces respectively.

Our contributions are as follows. Our theoretical framework encapsulate previous work that rely on the matching of feature distributions, such as~\cite{huang2017like,chen2021wasserstein,lohit2022model,lv2024wasserstein}. Then, we use theoretical results in domain adaptation~\cite{redko2017theoretical,redko2020survey}, for bounding the difference in the generalization error between student and teacher by the Wasserstein distance, and ultimately, by the difference in their encoder networks. Finally, in our experiments, we show that feature-based knowledge distillation improves over the student baseline in all scenarios, with a slight advantage towards probability metrics that consider the labels of samples.

This paper is organized as follows. Section~\ref{sec:prob_metrics} presents the probability metrics used in this work. Section~\ref{sec:kddm} presents the practical implementation of the minimization problem in equation~\ref{eq:kd_dm}. Section~\ref{sec:theoretical_results} presents our theoretical results. Section~\ref{sec:experiments} presents our empirical results. Finally, section~\ref{sec:conclusion} concludes this paper.

\section{Probability Metrics}\label{sec:prob_metrics}

We provide an overview of different metrics $\mathbb{D}$ between probability distributions. We focus on 2 types of models for feature distributions $P_{S} = g_{S,\sharp}P$ and $P_{T} = g_{S,\sharp}P$. The first, called empirical, assumes an i.i.d. sample $\{\mathbf{z}_{i}^{(P_{S})}\}_{i=1}^{n}$, obtained through $\mathbf{z}_{i}^{(P)} = g(\mathbf{x}_{i}^{(P)})$, so that,
\begin{align}
    \hat{P}_{S}(\mathbf{z}) = (g_{S,\sharp}\hat{P})(\mathbf{z}) = \dfrac{1}{n} \sum_{i=1}^{n}\delta(\mathbf{z} - \mathbf{z}_{i}^{(P)}),\label{eq:empirical_student}
\end{align}
and conversely for $\hat{P}_{T}$. If $\hat{P}_{S}$ is a distribution over $\mathcal{Z}\times\mathcal{Y}$, one can consider $\hat{P}(\mathbf{z}, y)$ supported on $\{\mathbf{z}_{i}^{(P)}, y_{i}^{(P)}\}$. Second, we consider Gaussian distributions,
\begin{align*}
    \hat{P}_{S}(\mathbf{z}) = \dfrac{1}{\sqrt{(2\pi)^{d}|\Sigma|}}\text{exp}\biggr(-\dfrac{1}{2}(\mathbf{z}-\hat{\mu})^{T}\hat{\Sigma}^{-1}(\mathbf{z}-\hat{\mu})\biggr).
\end{align*}
where $(\hat{\mu}, \hat{\Sigma})$ are estimated, using maximum likelihood, from $\{\mathbf{z}_{i}^{(P)}\}_{i=1}^{n}$. For feature-label joint distributions, we can instead model each conditional $\hat{P}_{S}(Z|Y)$ as a Gaussian, so that,
\begin{align*}
    \hat{P}_{S}(\mathbf{z},y) = \sum_{y \in \mathcal{Y}}P_{S}(Y=y)P_{S}(Z|Y=y)
\end{align*}

\subsection{Empirical Distributions}

For empirical distributions, we focus on \gls{ot} distances, especially the $2-$Wasserstein distance,
\begin{align}
    \mathbb{W}_{2}(\hat{P}_{S}, \hat{P}_{T})^{2} = \min{\gamma \in \Gamma(\hat{P}_{S}, \hat{P}_{T})}\sum_{i=1}^{n}\sum_{j=1}^{m}\lVert \mathbf{z}_{i}^{(P_{S})} - \mathbf{z}_{j}^{(P_{T})} \rVert_{2}^{2}\gamma_{ij},\label{eq:w2}
\end{align}
where $\Gamma(\hat{P}_{S}, \hat{P}_{T})$ is the set of transport plans. This distance measures the least effort for moving the distribution $\hat{P}_{S}$ into $\hat{P}_{T}$. We refer readers to~\cite{santambrogio2015optimal},~\cite{peyre2019computational} and~\cite{montesuma2024recent} for further details on \gls{ot} theory, its computational aspects, and recent applications in machine learning.

Classical probability metrics are limited, in the sense that they do not consider the label information of distributions. Here, we consider two ways of integrating labels into $\mathbb{D}$, namely, the class-conditional Wasserstein distance ($\mathbb{CW}_{2}$), and the joint Wasserstein distance ($\mathbb{JW}_{2}$). We now define the first one,
\begin{align}
    \mathbb{CW}_{2}(\hat{P}_{S},\hat{P}_{T})^{2} = \dfrac{1}{n_{c}}\sum_{y=1}^{n_{c}}\mathbb{W}_{2}(\hat{P}_{S}(Z|Y=y), \hat{P}_{T}(Z|Y=y))^{2},\label{eq:c_wasserstein}
\end{align}
where $n_{c}$ corresponds to the number of classes. Second, we can compare features and labels jointly:
\begin{align}
    \mathbb{JW}_{2}(\hat{P}_{S},\hat{P}_{T})^{2} = \min{\gamma \in \Gamma(\hat{P},\hat{Q})}\sum_{i=1}^{n}\sum_{j=1}^{m}\gamma_{ij}(\lVert \mathbf{z}_{i}^{(P_{S})} - \mathbf{z}_{j}^{(P_{T})} \rVert^{2} + \beta \mathcal{L}(h(\mathbf{z}_{i}^{(P_{S})}), h(\mathbf{z}_{j}^{(P_{T})})),\label{eq:j_wasserstein}
\end{align}
where $\mathcal{L}$ is a discrepancy between labels. Previous works have considered, for instance, the Euclidean distance~\cite{montesuma2023multi} between one-hot encoded vectors, or the cross-entropy loss~\cite{courty2017joint}. We refer readers to~\cite{montesuma2024thesis} for further details on this metric.

\subsection{Gaussian Distributions}

For Gaussian distributions, $\hat{P}_{S} = \mathcal{N}(\mu^{(P_{S})}, \Sigma^{(P_{S})})$ (resp. $\hat{P_{T}}$), we consider a metric, and a discrepancy. First, \gls{ot} has closed-form solution for Gaussian distributions, see~\cite{takatsu2011wasserstein} and~\cite[Ch. 8]{peyre2019computational}. In this case,
\begin{align}
    \mathbb{W}_{2}(\hat{P}_{S}, \hat{P}_{T})^{2} = \lVert \hat{\mu}^{(P_{S})} - \hat{\mu}^{(P_{T})} \rVert_{2}^{2} + \mathcal{B}(\hat{\Sigma}^{(P_{S})}, \hat{\Sigma}^{(P_{T})}).\label{eq:g_w2}
\end{align}
where $\mathcal{B}(A, B) = \text{Tr}(A) + \text{Tr}(B) - 2 \text{Tr}\left( \sqrt{ \sqrt{A} B \sqrt{A} }\right)$, and $\sqrt{A}$ denotes the square root matrix of $A$. When $\Sigma^{(P_{S})}$ and $\Sigma^{(P_{T})}$ are diagonal, this formula simplifies to $\mathbb{W}_{2}(P_{S}, P_{T})^{2} = \lVert 
\mu^{(P_{S})} - \mu^{(P_{T})} \rVert_{2}^{2} + \lVert \sigma^{(P_{S})} - \sigma^{(P_{T})} \rVert_{2}^{2}$, which induces an Euclidean geometry over $(\mu,\sigma) \in \mathbb{R}^{d}\times\mathbb{R}_{+}^{d}$.

Second, one can consider the Kullback-Leibler divergence~\cite{kullback1951information}, which has a closed-form for Gaussians as well,
\begin{align}
    \mathbb{KL}(\hat{P}_{S}|\hat{P}_{T}) = \dfrac{1}{2}\biggr(\text{Tr}((\hat{\Sigma}^{(P_{T})})^{-1}\hat{\Sigma}^{(P_{S})}) &+ (\hat{\mu}^{(P_{T})}-\hat{\mu}^{(P_{S})})^{T}(\hat{\Sigma}^{(P_{T})})^{-1}(\hat{\mu}^{(P_{T})}-\hat{\mu}^{(P_{S})})\nonumber\\ &- d + \log\biggr(\dfrac{\text{det}(\hat{\Sigma}^{(P_{T})})}{\text{det}(\hat{\Sigma}^{(P_{S})})}\biggr)\biggr),\label{eq:g_kl}
\end{align}
which, for diagonal covariance matrices, simplifies to,
\begin{align*}
    \mathbb{KL}(\hat{P}_{S}|\hat{P}_{T}) = \biggr\lVert \dfrac{\sigma^{(P_{S})}}{\sigma^{(P_{T})}} \biggr\rVert_{2}^{2} + \biggr\lVert \dfrac{\mu^{(P_{T})} - \mu^{(P_{S})}}{\sigma^{(P_{T})}} \biggr\rVert_{2}^{2} - d + 2\sum_{i=1}^{d}\log\biggr{(}\dfrac{\sigma_{i}^{(P_{T})}}{\sigma^{(P_{S})}_{i}}\biggr{)}.
\end{align*}
In this case, $\mathbb{KL}$ is associated with a hyperbolic geometry in the space of parameters, see~\cite[Ch. 8, Remark 8.2]{peyre2019computational} and~\cite[Section 2]{montesuma2024recent}.

\section{Knowledge Distillation through Distribution Matching}\label{sec:kddm}

Following our discussion in the last section, we now describe a practical implementation for minimizing the distance, in distribution, between student and teacher features. We assume a fixed distribution $P$, from which we have access to samples $\{\mathbf{x}_{i}^{(P)}, y_{i}^{(P)}\}_{i=1}^{n}$, $\mathbf{x}_{i}^{(P)} \in \mathcal{X}$ and $y_{i}^{(P)} \in \mathcal{Y}$.

\begin{algorithm}[ht]
  \caption{Training step of KD$^{2}$M}
  \label{alg:kddm}
  \Function{training\_on\_minibatch($\{\mathbf{x}_{i}^{(P)}, y_{i}^{(P)}\}_{i=1}^{n}, \lambda$)}{
    \Comment{Forward pass - Student}
    $\mathbf{Z}^{(P_{S})} \leftarrow \{g_{S}(\mathbf{x}_{i}^{(P)})\}_{i=1}^{n}$, and, $\hat{\mathbf{Y}}^{(P_{S})} \leftarrow \{h_{S}(\mathbf{z}_{i}^{(P_{S})})\}$\;
    \Comment{Classification loss - Student}
    $\mathcal{L}_{c} \leftarrow -\dfrac{1}{n}\sum_{i=1}^{n}\sum_{y=1}^{n_{c}}y_{ic}^{(P)}\log \hat{y}_{ic}^{(P_{S})}$\;
    \Comment{Forward pass - Teacher}
    $\mathbf{Z}^{(P_{T})} \leftarrow \{g_{T}(\mathbf{x}_{i}^{(P)})\}_{i=1}^{n}$, and, $\hat{\mathbf{Y}}^{(P_{T})} \leftarrow \{h_{T}(\mathbf{z}_{i}^{(P_{T})})\}$\;
    \Comment{Feature distillation loss}
    $\mathcal{L}_{d} \leftarrow \text{compute\_distribution\_distance}(\mathbf{Z}^{(P_{S})}, \mathbf{Z}^{(P_{T})}, \mathbf{Y}^{(P)}, \hat{\mathbf{Y}}^{(P_{S})}, \hat{\mathbf{Y}}^{(P_{T})})$\;
    \Return{$\mathcal{L}_{c} + \lambda \mathcal{L}_{d}$}\;
  }
\end{algorithm}

Denoting the empirical approximation of $P$ by $\hat{P}$, we obtain the student and teacher distributions through the push-forward of $\hat{P}$ through each encoder network, i.e., $\hat{P}_{S} = g_{S,\sharp}\hat{P}$ (resp. $\hat{P}_{T}$), as we show in equation~\ref{eq:empirical_student}. We match $\hat{P}_{S}$ with $\hat{P}_{T}$ by minimizing some probability metric $\mathbb{D}$ with respect to the parameters of $g_{S}$. In practice, we use \gls{sgd}, that is, we draw mini batches from the distribution $P$, feed them through the training student network and the frozen target distribution, then, we compute the gradients through automatic differentiation~\cite{paszke2017automatic}. The mismatch comes from the fact that $g_{S} \neq g_{T}$, therefore, the feature distributions $g_{S,\sharp}P$ and $g_{T,\sharp}P$ are, in general, different. We call our proposed method \gls{kddm}, shown in Algorithm~\ref{alg:kddm}.

Recent work has used the Wasserstein distance, especially in its empirical~\cite{chen2021wasserstein,lohit2022model} (equation~\ref{eq:w2}) and Gaussian formulations~\cite{lv2024wasserstein} (equation~\ref{eq:g_w2}) for \gls{kd} of feature distributions. Likewise,~\cite{huang2017like} used \gls{mmd} for the same purpose. The matching of class-conditional distributions (e.g. equations~\ref{eq:c_wasserstein}) has been widely used in the related field of \emph{dataset distillation}, which aims at compressing datasets with respect to their number of samples~\cite{liu2023dataset,zhao2023dataset}. The joint Wasserstein distance (equation~\ref{eq:j_wasserstein}) has been used to this end~\cite{montesuma2024multi}. Finally, the Kullback-Leibler divergence is a cornerstone of \gls{kd}, used to compare neural net predictions~\cite{hinton2015distilling} and feature distributions~\cite{lv2024wasserstein}.


\section{Theoretical Results}\label{sec:theoretical_results}

In this section, we draw on our framework for deriving error bounds for \gls{kd}. These bounds are inspired by the domain adaptation theory~\cite{redko2020survey}, which deals with questions similar to those of \gls{kddm}. Before proceeding, we recall that, for a measurable function $g$,
\begin{align}
    \lVert g \rVert_{L_{2}(P)}^{2} = \int_{\mathcal{Z}} |g(\mathbf{z})|^{2}dP(\mathbf{z})\text{, and, }(g_{\sharp}P)(A) = P(g^{-1}(A)).\label{eq:definitions}
\end{align}
For simplicity, we state our results for $\mathcal{Z}$, which can be understood as a sub-set of $\mathbb{R}^{d}$, that is, we are using $d-$dimensional feature vectors. Furthermore, we denote by $\mathcal{P}(\mathcal{Z})$ the set of distributions over $\mathcal{Z}$.

We begin our analysis with the definition of generalization error,

\begin{definition}{(Error)}
Given a $P \in \mathcal{P}(\mathcal{X})$, a loss function $\mathcal{L}:\mathcal{Y}\times\mathcal{Y}\rightarrow \mathbb{R}_{+}$, and a ground truth $f_{0}:\mathcal{X} \rightarrow \mathcal{Y}$, the generalization error of $f$ is,
\begin{align*}
    \mathcal{R}_{P}(f)=\mathbb{E}_{x \sim P}[\mathcal{L}(f(x), f_{0}(x))]
\end{align*}
\end{definition}

For neural nets, $f = g \circ h$, where $g:\mathcal{X}\rightarrow\mathcal{Z}$ is a feature extraction, and $h:\mathcal{Z}\rightarrow\mathcal{Y}  $ is a feature classifier. Naturally, the same definition applies over $\mathcal{Z}$ instead of $\mathcal{X}$, by considering $\mathcal{L}(h(z), h_{0}(z))$.

In the following we consider the risk over extracted features, i.e., we consider distributions over the latent space $\mathcal{Z}$. Next, we use a result from domain adaptation~\cite[Lemma 1]{redko2017theoretical} for bounding the generalization error of the student by the generalization error of the teacher.

\begin{lemma}{\cite{redko2017theoretical}}
    Let $\mathcal{Z} \subset \mathbb{R}^{d}$ be separable. Let $P_{S}, P_{T} \in \mathcal{P}(\mathcal{Z})$. Assume $c(\mathbf{z}, \mathbf{z}') = \lVert \mathbf{z} - \mathbf{z}' \rVert_{\mathcal{H}_{k}}$, where $\mathcal{H}_{k}$ is a reproducing kernel Hilbert space with kernel $k:\mathcal{Z}\times\mathcal{Z}\rightarrow\mathbb{R}$ induced by $\phi:\mathcal{Z}\rightarrow\mathcal{H}_{k}$. Assume that $\mathcal{L}_{h,h'}(\mathbf{z}) = |h(\mathbf{z}) - h'(\mathbf{z})|$, and that $k$ is squared root integrable with respect $P_{S}$ and $P_{T}$, and $0 \leq k(\mathbf{z},\mathbf{z}') \leq K, \forall \mathbf{z},\mathbf{z'} \in \mathcal{Z}$. Assuming $\lVert 
\mathcal{L} \rVert_{\mathcal{H}_{k}} \leq 1$,
    \begin{equation}
        | \mathcal{R}_{P_{S}}(h) - \mathcal{R}_{P_{T}}(h) | \leq \mathbb{W}_{2}(P_{S}, P_{T}).\label{eq:w2_bound}
    \end{equation}
\end{lemma}
\noindent\textbf{Discussion.} The previous lemma holds uniformly for all hypothesis $h:\mathcal{Z}\rightarrow\mathcal{Y}$. In this sense, the r.h.s. does not explicitly depends on $h$, as the bound is obtained for the worst-case scenario. We refer readers to~\cite{redko2017theoretical} and~\cite{redko2020survey} for more details.

Based on this lemma, we leverage the special structure between the student and teacher feature distributions, that is, $P_{S} = g_{S,\sharp}P$ and $P_{T} = g_{T,\sharp}P$, which we now present,
\begin{theorem}\label{thm:main_result}
    Under the same conditions of Lemma 1, let $P \in \mathcal{P}(\mathcal{Z})$ be a fixed distribution. Let $g_{S}$ and $g_{T}$ be two measurable mappings from $\mathcal{X}$ to a latent space $\mathcal{Z} \subset \mathbb{R}^{d}$, such that $\lVert g_{S} \rVert_{L_{2}(P)} < \infty$ and $\lVert g_{T} \rVert_{L_{2}(P)} < \infty$. Define $P_{S} = g_{S,\sharp}P$ and $P_{T} = g_{T,\sharp}P$, then,
    \begin{align}
        | \mathcal{R}_{P_{S}}(h) - \mathcal{R}_{P_{T}}(h) | \leq \lVert g_{S} - g_{T} \rVert_{L_{2}(P)}\label{eq:l2_bound}
    \end{align}
\end{theorem}
\begin{proof}
    Our goal is to show $\mathbb{W}_{2}(g_{S,\sharp}P, g_{T,\sharp}P) \leq \lVert g_{S} - g_{T} \rVert_{L_{2}(P)}$, from which equation~\ref{eq:l2_bound} follows from equation~\ref{eq:w2_bound}. To establish the result, note that,
    \begin{align*}
        \mathbb{W}_{2}(g_{S,\sharp}P,g_{T,\sharp}P)^{2} = \inf{\gamma \in \Gamma(P, P)}\int \lVert g_{S}(x) - g_{T}(x')\rVert_{2}^{2}d\gamma(x,x').
    \end{align*}
    Choose $\gamma(x, x') = P(x)\delta(x - x') \in \Gamma(P, P)$. Since $\mathbb{W}_{2}$ is as an infimum over $\gamma$,
    \begin{align*}
        \mathbb{W}_{2}(g_{S,\sharp}P,g_{T,\sharp}P)^{2} \leq \int\lVert g_{S}(x) - g_{T}(x')\rVert_{2}^{2}P(x)dx = \lVert g_{S} - g_{T} \rVert_{L_{2}(P)}^{2}
    \end{align*}
\end{proof}

\noindent\textbf{Discussion.} This theorem bounds the risks via the $L_{2}(P)$ distance of the encoder networks $g_{S}$ and $g_{T}$. This result is possible precisely because $P_{S}$ and $P_{T}$ are the image distributions of $P$ via $g_{S}$ and $g_{T}$. These bounds lead to an interesting consequence. If $g_{S} \rightarrow g_{T}$ uniformly, then $\lVert g_{S} - g_{T} \rVert_{L_{2}(P)} \rightarrow 0$ and $\mathbb{W}_{2}(g_{S,\sharp}P, g_{T,\sharp}P) \rightarrow 0$, which implies $\mathcal{R}_{P_{S}}(h) \rightarrow \mathcal{R}_{P_{T}}(h)$, i.e., both networks achieve the same generalization error.

\section{Experiments}\label{sec:experiments}

    \begin{wraptable}[19]{r}{0.5\textwidth}  
    \centering
    \caption{Classification accuracy (in \%) of \gls{kddm} for different distribution metrics on computer vision benchmarks. Distances are either over empirical (E) or Gaussian (G) approximations.}
    \resizebox{\linewidth}{!}{%
        \begin{tabular}{lcccc}
            \toprule
            Method & SVHN & CIFAR-10 & CIFAR-100 & Avg. \\
            \midrule
            Student & ResNet18 & ResNet18 & ResNet18 & -- \\
            Teacher & ResNet34 & ResNet34 & ResNet34 & -- \\
            \midrule
            Student & 93.10 & 85.11 & 56.66 & 78.29 \\
            Teacher & 94.41 & 86.98 & 62.21 & 81.20 \\
            \midrule
            $\mathbb{W}_{2}$ (E) & 94.00 & 86.45 & 61.07 & 80.51 \\
            $\mathbb{CW}_{2}$ (E) & \textbf{94.06} & 86.54 & \textbf{61.47} & \textbf{80.69} \\
            $\mathbb{JW}_{2}$ (E) & 94.00 & \textbf{86.60} & 61.07 & 80.55 \\
            \midrule
            $\mathbb{W}_{2}$ (G) & 93.94 & 86.63 & 60.68 & 80.41 \\
            $\mathbb{CW}_{2}$ (G) & 93.95 & 86.25 & 61.43 & 80.54 \\
            $\mathbb{KL}$ (G) & 94.05 & 86.44 & 60.66 & 80.38 \\
            \bottomrule
        \end{tabular}%
    }
    \label{tab:results-kd}
\end{wraptable}
In this section, we empirically validate our \gls{kddm} framework. We experiment with datasets in computer vision, namely SVHN~\cite{netzer2011reading}, CIFAR-10, and CIFAR-100~\cite{krizhevsky2009learning}. For the backbones, we use ResNet~\cite{he2016deep} with 18 and 34 layers (student and teacher, respectively). Results are summarized in Table~\ref{tab:results-kd}.

In all of our experiments, we optimize ResNets with a SGD optimizer with learning rate $0.01$ and $0.9$ momentum factor for 15 epochs. Here, an epoch is defined as a full passage through the dataset. At the end of each epoch, we schedule the learning rate using Cosine annealing\footnote{\url{https://pytorch.org/docs/stable/generated/torch.optim.lr_scheduler.CosineAnnealingLR.html}}~\cite{loshchilov2016sgdr}, with a minimum learning rate of $10^{-4}$.

All benchmarks are composed of $32 \times 32 \times 3$ RGB images. SVHN is composed of 600000 images of printed digits ($0$ to $9$). CIFAR10 and CIFAR100 are composed of 60000 images each. For CIFAR10, we have 10 classes, whereas CIFAR100 has 100 classes.

In all three datasets, we use data augmentation. For SVHN we use random cropping, random horizontal flip, and we then normalize images over each channel using $\mu = (0.4377, 0.4438, 0.4728)$ and $\sigma = (0.1980, 0.2010, 0.1970)$. For CIFAR10 and 100, we use the same procedure, except that we use $\mu = (0.5071, 0.4867, 0.4408)$ and $\sigma = (0.2675, 0.2565, 0.2761)$.

Feature distillation improves over the student baseline in \emph{all} experiments. Our theoretical results (Theorem~\ref{thm:main_result}) show the student's performance is upper-bounded by the teacher's, which we also verify in practice. Distillation methods perform similarly, with slight advantages for label-aware metrics, as they better match feature-label joint distributions.

For SVHN, we divide our analysis in two. First, in Fig.~\ref{fig:performances_svhn}, we show a comparison of the classification and distillation losses, as well as accuracy per training epoch for various $\lambda$ (see equation~\ref{eq:kd_dm}). This remark shows that for reasonable values of $\lambda$, \gls{kddm} improves over the baseline. Second, in Fig.~\ref{fig:comparison-tsne-svhn}, we analyze the distribution of extracted features for the teacher (red points) and the student (blue points). Through \gls{kddm}, we achieve better alignment of these features.

\begin{figure}[ht]
    \centering
    \begin{subfigure}{0.32\linewidth}
        \centering
        \includegraphics[width=\linewidth]{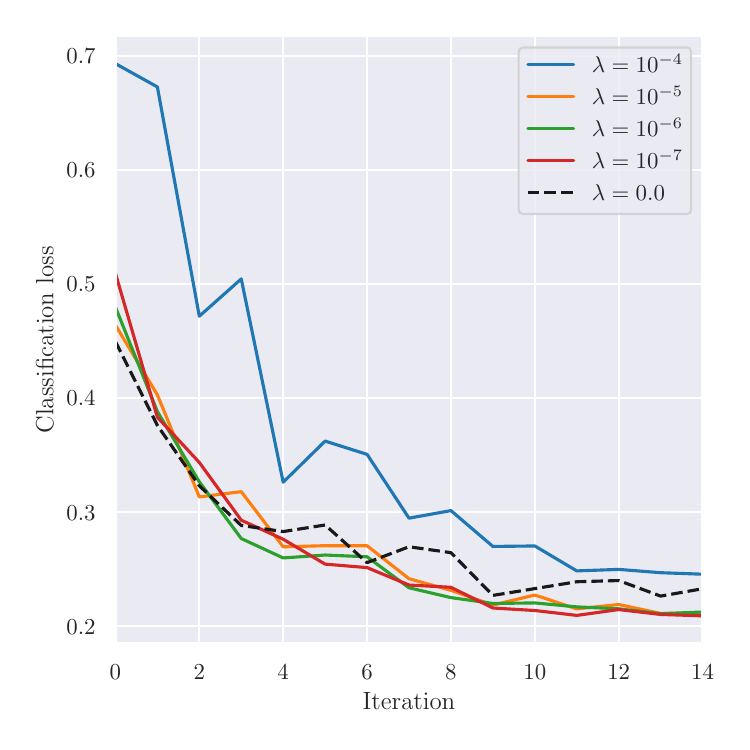}
    \end{subfigure}\hfill
    \begin{subfigure}{0.32\linewidth}
        \centering
        \includegraphics[width=\linewidth]{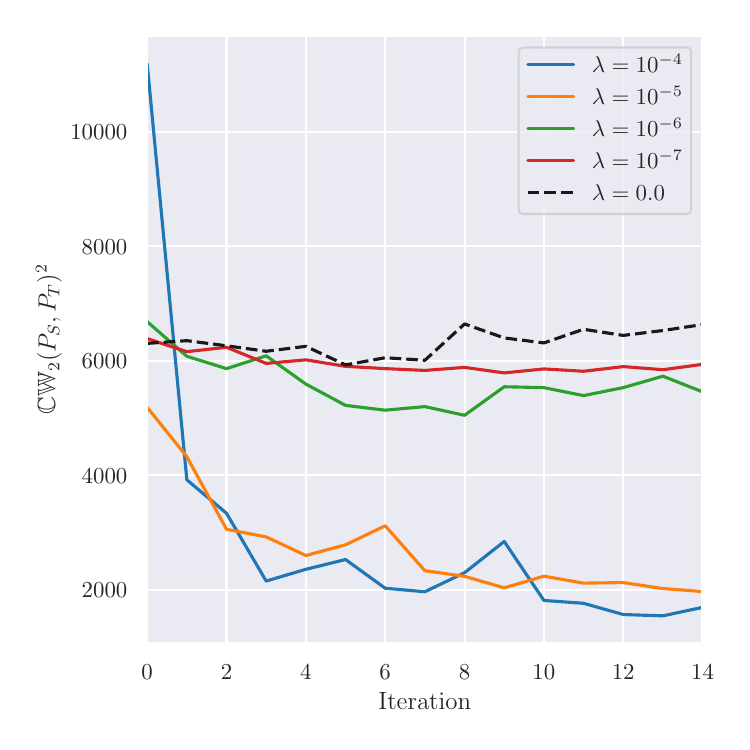}
    \end{subfigure}\hfill
    \begin{subfigure}{0.32\linewidth}
        \centering
        \includegraphics[width=\linewidth]{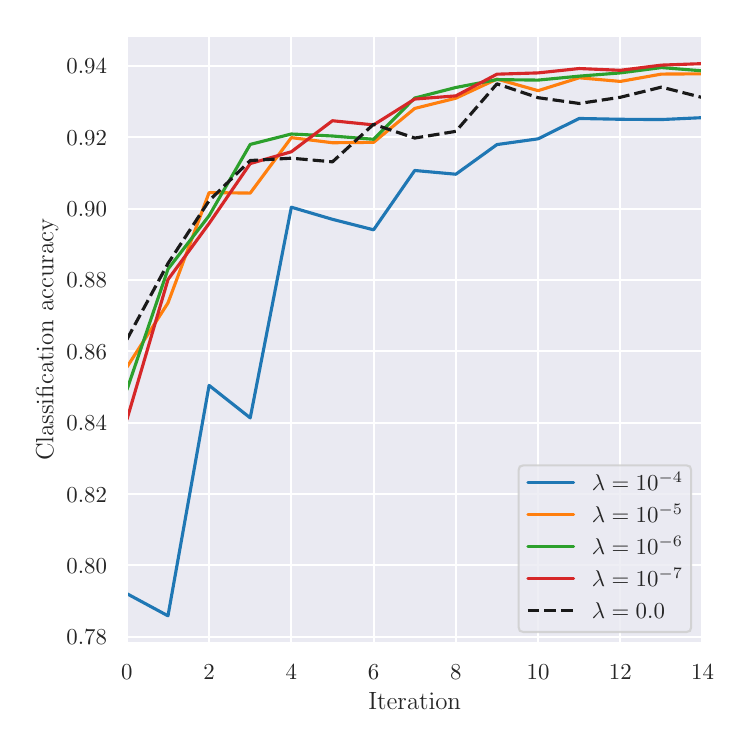}
    \end{subfigure}\hfill
    \caption{Classification and distillation losses, alongside accuracy on the SVHN benchmark. Using \gls{kddm} improves over the no distillation baseline.}
    \label{fig:performances_svhn}
\end{figure}

\begin{figure}[ht]
    \centering
    \includegraphics[width=\linewidth]{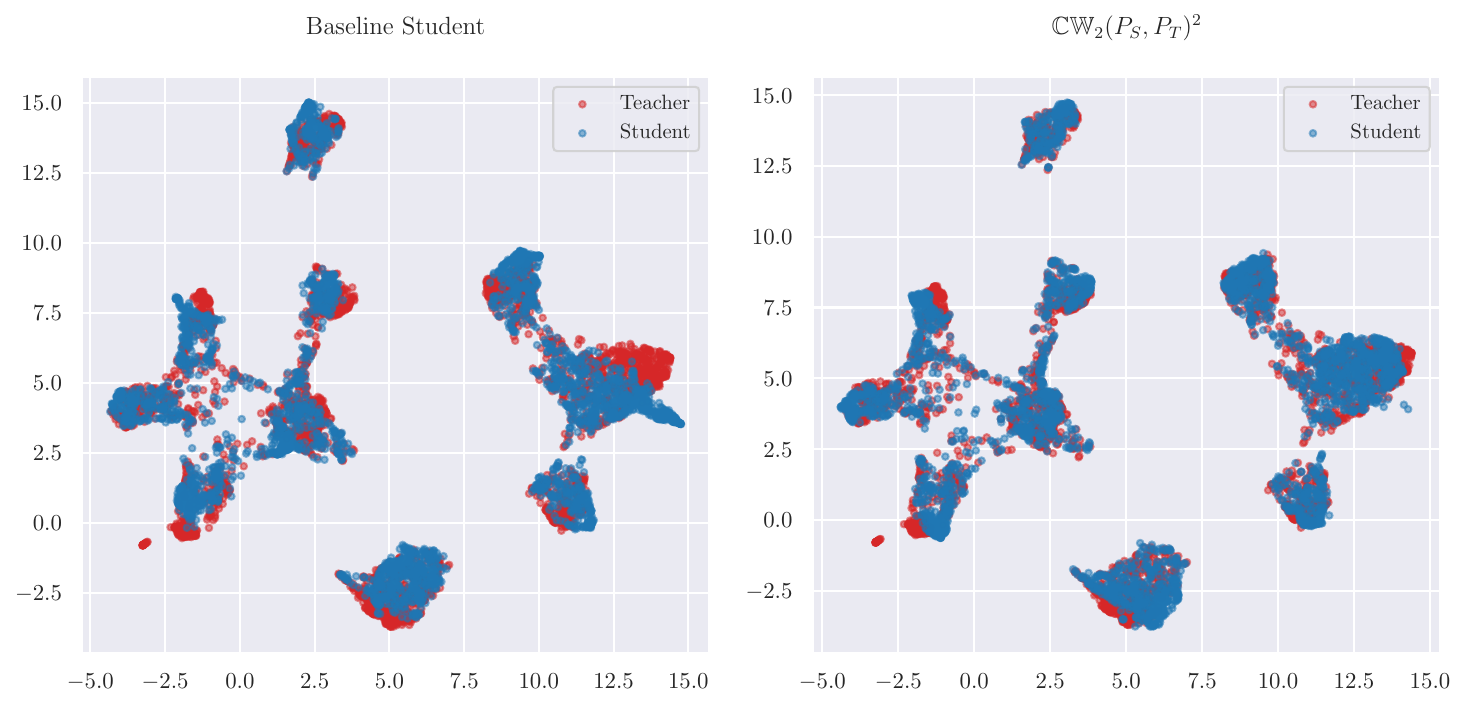}
    \caption{Feature alignment analysis of the baseline student (left), and the \gls{kddm} student with the $\mathbb{CW}$ metric (right). Teacher features are shown in red, whereas student features are shown in blue. Overall, \gls{kddm} better aligns teacher and student features.}
    \label{fig:comparison-tsne-svhn}
\end{figure}

\section{Conclusion}\label{sec:conclusion}

In this paper, we propose a novel, general framework for \gls{kd}, based on the idea of distribution matching. This principle seeks to align student and teacher encoder networks, via a term depending on the probability distribution of its features. We theoretically motivate this approach using results from domain adaptation~\cite{redko2017theoretical}, which show that, under suitable conditions, the performance of the teacher and student network is bounded by the distribution matching term, as well as the difference between student and target networks. We further validate our framework empirically, using various known probability metrics or discrepancies, such as the Wasserstein distance and the Kullback-Leibler divergence, showing that these methods improve over the simple optimization of the student network over training data. Overall, our work opens new possibilities for the theory of \gls{kd}, via connections with domain adaptation theory.

\bibliographystyle{apalike}
\bibliography{refs.bib}

\end{document}